\documentclass[final]{cvpr}

\usepackage{times}  
\usepackage{epsfig}
\usepackage{graphicx}
\usepackage{amsmath}
\usepackage{amssymb}
\usepackage{bm}
\usepackage{multirow}
\usepackage{booktabs}
\usepackage[table]{xcolor}
\usepackage{caption}
\usepackage{algorithm}
\usepackage{algorithmic}

\usepackage[pagebackref=true,breaklinks=true,colorlinks,bookmarks=false]{hyperref}

\def\cvprPaperID{6367} 
\def\confYear{CVPR 2021}

\newcommand{\x}{{\bf x}}

\newcommand{\btheta}{{\bm \theta}}
\newcommand{\bomega}{{\bm \omega}}
\newcommand{\bmu}{{\bm \mu}}

\newcommand{\bbP}{\mathbb P}
\newcommand{\bbR}{\mathbb R}
\newcommand{\trn}{\text{trn}}

\newcommand{\bbE}{\mathbb E}
\newcommand{\cX}{\mathcal X}
\newcommand{\cH}{\mathcal H}
\newcommand{\cB}{\mathcal B}
\newcommand{\cD}{\mathcal D}
\newcommand{\cF}{\mathcal F}
\newcommand{\cO}{\mathcal O}
\newcommand{\cY}{\mathcal Y}
\newcommand{\cS}{\mathcal S}
\newcommand{\fP}{\mathfrak P}
\newcommand{\fK}{\mathfrak K}
\newcommand{\cP}{\mathcal P}
\newcommand{\daerm}{DA-ERM }
\newcommand{\dammd}{DA-MMD }
\newcommand{\dacoral}{DA-CORAL}
\newcommand{\ltin}{LT-ImageNet}
\newcommand{\ltinp}{LT-ImageNet}

\DeclareMathOperator*{\argmin}{arg\,min}

\usepackage{amsthm}

\newtheorem{theorem}{Theorem}

\newtheorem{lemma}{Lemma}

\newtheorem{remark}{Remark}

\usepackage{mathtools}

\begin{document}

\title{Adaptive Methods for Real-World Domain Generalization}

\author{Abhimanyu Dubey\thanks{Work done while visiting Facebook AI.}\\
MIT\\
{\tt\small dubeya@mit.edu}
\and
Vignesh Ramanathan\\
Facebook AI\\
{\tt\small vigneshr@fb.com}
\and
Alex Pentland\\
MIT\\
{\tt\small pentland@mit.edu}
\and
Dhruv Mahajan\\
Facebook AI\\
{\tt\small dhruvm@fb.com}
}
\maketitle

\begin{abstract}
Invariant approaches have been remarkably successful in tackling the problem of domain generalization, where the objective is to perform inference on data distributions different from those used in training. In our work, we investigate whether it is possible to leverage domain information from the unseen test samples themselves. We propose a domain-adaptive approach consisting of two steps: a) we first learn a discriminative domain embedding from unsupervised training examples, and b) use this domain embedding as supplementary information to build a domain-adaptive model, that takes both the input as well as its domain into account while making predictions. For unseen domains, our method simply uses few unlabelled test examples to construct the domain embedding. This enables adaptive classification on any unseen domain. Our approach achieves state-of-the-art performance on various domain generalization benchmarks. In addition, we introduce the first real-world, large-scale domain generalization benchmark, Geo-YFCC, containing $1.1$M samples over $40$ training, $7$ validation and $15$ test domains, orders of magnitude larger than prior work. We show that the existing approaches either do not scale to this dataset or underperform compared to the simple baseline of training a model on the union of data from all training domains. In contrast, our approach achieves a significant $1\%$ improvement.
\end{abstract}
\setlength{\tabcolsep}{2pt}
\section{Introduction}
\begin{figure}[t!]
\centering
\includegraphics[width=0.7\linewidth]{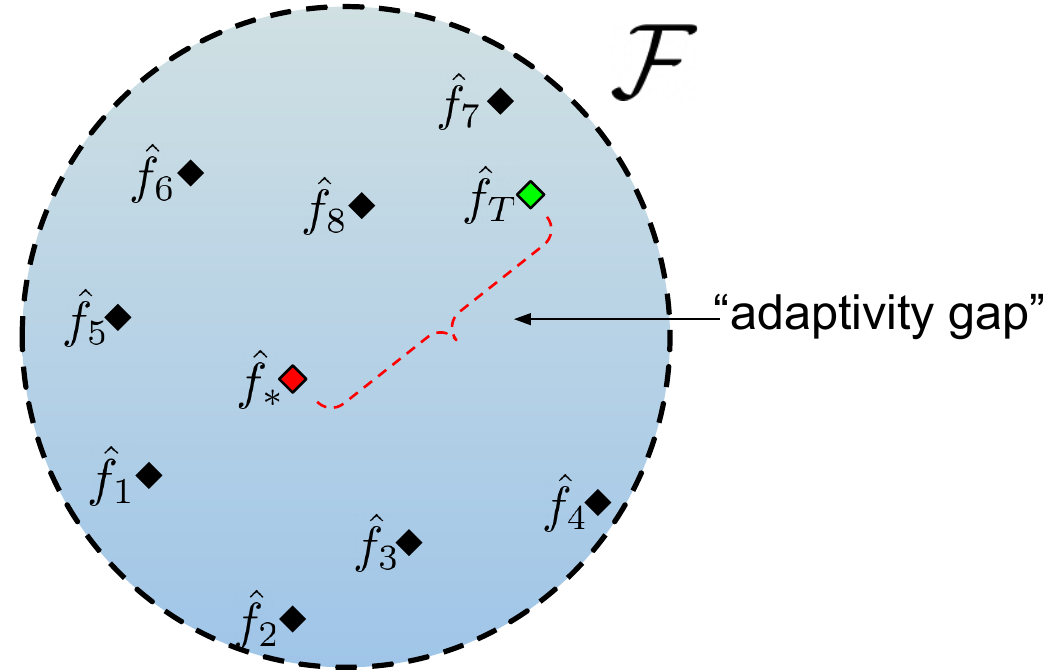}
\caption{A visual for \textit{adaptive} domain generalization. Here we denote by $\hat f_i$ the optimal classifier in a class $\cF$ for domain $i$ (out of $8$). For test domain $D_T$, the optimal \textit{universal} classifier $\hat f_*$ may be far from the optimal classifier for $D_T$.}\label{fig:intuition}
\end{figure}

Domain generalization refers to the problem of learning a classifier from a heterogeneous collection of distinct \textit{training} domains that can generalize to new unseen \textit{test} domains~\cite{blanchard2011generalizing}. Among the various formalizations proposed for the problem~\cite{gulrajani2020search, ben2010theory}, the most effective is \textit{domain-invariant} learning~\cite{baktashmotlagh2013unsupervised, ganin2016domain}. It learns feature representations invariant to the underlying domain, providing a \textit{universal} classifier that can generalize to new domains~\cite{li2018deep, li2018deep, muandet2013domain, li2018domain, ghifary2015domain}. 

It can be demonstrated that domain-\textit{invariant} classifiers minimize the \textit{average risk} across domains~\cite{muandet2013domain}. However, this may not guarantee good performance for any \textit{specific} test domain, particularly if the distribution of domains has high variance, as visualized heuristically in Figure~\ref{fig:intuition}. In this case, the optimal classifier for a target domain can lie far from the optimal \textit{universal} classifier. In our work, we propose an \textit{adaptive} classifier that can be adapted to any new domain using very few unlabelled samples without any further training. Unlike invariant approaches, this requires a few \textit{unsupervised} samples from any domain while \textit{testing}\footnote{This does not correspond to unsupervised domain adaptation, where unsupervised samples are assumed to be present during \textit{training}.}. However, this requirement is trivial, since this set is available by definition in all practical settings.

Our approach consists of two steps. We first embed each domain into a vector space using very few unlabelled samples from the domain. Next, we leverage these \textit{domain embeddings} as supplementary signals to learn a \textit{domain-adaptive} classifier. During testing, the classifier is supplied the corresponding embedding obtained from test samples. Our contributions can be summarized as follows.

1. We adapt low-shot \textit{prototypical learning}~\cite{snell2017prototypical} to construct \textit{domain embeddings} from unlabelled samples of each domain. We also provide an algorithm to use these embeddings as supplementary signals to train \textit{adaptive} classifiers.

2. We justify our design choices with novel theoretical results based on the framework of kernel mean embeddings~\cite{gretton2012kernel}. Furthermore, we leverage the theory of risk minimization under product kernel spaces~\cite{blanchard2011generalizing} to derive generalization bounds on the average risk for adaptive classifiers.

3. We introduce the first large-scale, real-world domain generalization benchmark, dubbed \textbf{Geo-YFCC}, which contains $40$ training, $7$ validation and $15$ test domains, and an overall $1.1M$ samples. Geo-YFCC is constructed from the popular YFCC100M~\cite{thomee2016yfcc100m} dataset partitioned by geographical tags, and exhibits several characteristics of \textit{domain shift}~\cite{li2017deeper}, \textit{label shift}~\cite{azizzadenesheli2019regularized}, and long-tailed class distributions~\cite{zhang2017range, xiang2020learning, cao2020domain} common in real-world classification. Because of its scale and diversity, this benchmark is substantially more challenging than the incumbent benchmarks.

4. On existing benchmarks and notably, two large-scale benchmarks, we demonstrate the strength of fine-tuning (ERM) with \textit{adaptive} classification over existing approaches. In addition to the effectiveness of adaptive classification, this suggests, along with the claims of Gulrajani and Lopez-Paz~\cite{gulrajani2020search} that rigorous model selection and large benchmarks are essential in understanding domain generalization, where naive ERM can also be a powerful baseline.


\section{Related Work}
Our paper is inspired by several areas of research, and we enumerate the connections with each area sequentially.\vspace*{0.05in}\\
\noindent \textbf{Domain Generalization.} The problem of domain generalization was first studied as a variant of multi-task learning in Blanchard~\etal\cite{blanchard2011generalizing}. Many \textit{domain-invariant} approaches~\cite{sun2016deep, li2018domain, li2018deep}, have been developed using the popular domain adaptation algorithm introduced in~\cite{ganin2016domain}, where the authors learned invariant representations via adversarial training. This has been followed by alternative formulations for invariant learning such as MMD minimization~\cite{li2018deep}, correlation alignment~\cite{sun2016deep}, and class-conditional adversarial learning~\cite{li2018domain}. Other approaches include meta learning~\cite{li2017learning}, invariant risk minimization~\cite{arjovsky2019invariant}, distributionally robust optimization~\cite{sagawa2019distributionally}, mixup~\cite{xu2020adversarial, yan2020improve, wang2020heterogeneous}, and causal matching~\cite{mahajan2020domain}. \textit{Adversarial training} with improvements~\cite{ganin2016domain, pei2018multi} has also been used to learn invariant representations.
Complementary to these approaches, we focus instead on learning \textit{adaptive} classifiers that are specialized to each target domain\footnote{Note that our algorithm is not opposed to invariant learning, we merely focus on learning \textit{distribution-adaptive} classifiers, that can be combined with invariant approaches with ease, see Section~\ref{sec:smallscale}.}. Our approach also does not use any unsupervised data from unseen domains during training, as is done in the problem of unsupervised domain adaptation~\cite{gong2012geodesic, baktashmotlagh2013unsupervised}.\vspace*{0.05in}\\
\noindent \textbf{Domain-Adaptive Learning.} The idea of using \textit{kernel mean embeddings} (KME) for adaptive domain generalization was proposed in the work of Blanchard~\etal\cite{blanchard2011generalizing}. Kernel mean embeddings have also been used for personalized learning in both multi-task~\cite{deshmukh2017multi} and multi-agent learning~\cite{dubey2020kernel} bandit problems. A rigorous treatment of \textit{domain-adaptive} generalization in the context of KME approaches is provided in Deshmukh~\etal\cite{deshmukh2019generalization}. Our approach complements this theoretical line of work, by providing an efficient algorithm for classification, that requires only a few unsupervised samples for competitive performance. We also extend the results of~\cite{blanchard2011generalizing} to a larger class of functions.\vspace*{0.05in}\\
\noindent \textbf{Large-Scale Learning.} We propose domain generalization benchmarks based on both the ImageNet LSVRC12~\cite{deng2009imagenet} and YFCC100M~\cite{thomee2016yfcc100m} datasets, which have been instrumental in accelerating research for image classification. We consider the challenges encountered specifically for large-scale computer vision, which is plagued by issues such as long-tailed data distributions~\cite{lin2020distribution}, and large training times~\cite{abu2016youtube}. It is important to note that in large-scale settings, it is difficult to perform extensive hyper-parameter tuning and optimal model selection owing to large training times.
\section{Approach}

We assume that each \textit{domain} $D \in \cD$ is a probability distribution over $\cX \times \cY$, i.e., the product of the input space $\cX$ and output space $\cY$, and there exists a \textit{mother} distribution $\fP$ which is a measure over the space of domains $\cD$. A \textit{training domain} $\widehat D(n)$ is obtained by first sampling a domain $D$ from $\fP$, and then sampling $n$ points ($X$ and $Y$) from $\cX \times \cY$ following $D$. A training set can then be constructed by taking the union of $N$ such domains $(\widehat{D}_i(n))_{i=1}^N$. Correspondingly, any test domain $\widehat D_T(n_T)$ is also constructed by first drawing a sample $D_T \sim \cD$, then drawing $n_T$ samples from $D_T$, and discarding the labels.



Consider a family of functions $\cF$. For each $D \in \cD$, the optimal classifier $f_D$ in $\cF$ (for some loss function $\ell$) can be given by $f_D = \argmin_{f \in \cF} \bbE_{(\x, y) \sim D}[\ell(f(\x), y)]$. We denote the optimal empirical risk minimization (ERM) classifier for each domain $\widehat D_i$ in the training set domain $D_i(n)$ as $\hat f_i = \argmin_{f \in \cF} \bbE_{(\x, y) \sim  D(n)}[\ell(f(\x), y)]$. The \textit{universal} expected risk minimizer (i.e., classifier minimizing overall risk over $D \sim \fP$) can then be given by $f_*$ such that,
\begin{equation}
f_* = \argmin_{f \in \cF} \bbE_{D \sim \fP}\bbE_{\x, y \sim D}[\ell(f(\x), y)].
\label{eq:erm}
\end{equation}
Typically, the optimal classifier for $D_T$ within $\cF$ may be far from the (non-adaptive) universal classifier $\hat f_*$, a distance which we call the ``adaptivity gap''. A visual heuristic representation of the above intuition is provided in Figure~\ref{fig:intuition}. Ideally, we would like the classifier to be chosen closer to $\hat f_T$ such that it obtains lower error. Note that the \textit{adaptivity gap} may be large even when the training set contains samples close to $\hat f_T$ (in the figure, this refers to functions $\hat f_7$ and $\hat f_8$), since the \textit{universal} classifier minimizes overall risk.

Our approach reduces this discrepancy by learning an \textit{adaptive} classifier, i.e., a function $F: \mathcal X \times \mathcal D \rightarrow \mathcal Y$ that takes \textit{both} the point $\x$ and the domain $D$ of $\x$ into consideration. Our goal is to make sure that for any test domain $D_T$, the resulting classifier $f_T = F(\cdot, D_T)$ from $F$ lies close to $\hat f_T$. We assume that a small set of unsupervised data is available from $D_T$ to characterize $F(\cdot, D_T)$ from $F$\footnote{This assumption is trivially satisfied by the test set itself. Note that we do not access $D_T$ during training and hence we do not perform any \textit{transductive learning}, unlike unsupervised domain adaptation.}.
We follow a two step approach:

\noindent{(A) \textbf{Computing domain embeddings.}} We first learn a function  $\bmu : D \rightarrow \bbR^{d_\cD}$, that maps any domain $D \in \cD$ (including empirical domains) to a finite vector (domain embedding). Ideally, we want this embedding to be learned from a few samples belonging to the domain. To achieve this, we leverage kernel mean embeddings~\cite{muandet2016kernel}, a formalization that can represent probability distributions as the average of some feature vectors $\Phi_\cD$ in a suitable Hilbert space. $\bmu$ is obtained as the feature $\Phi_\cD(\x)$ averaged over all $\x \in D$.


\noindent{(B) \textbf{ERM using augmented inputs.}} With a suitable domain embedding $\bmu$, we learn a neural network $F$ directly over the \textit{augmented input} $(\x, \bmu(\widehat D))$, where $\x$ is a point from the domain $\widehat D$, using regular ERM (i.e., minimizing cross-entropy loss from predictions, see Algorithm~\ref{alg:main_algo}). 

\noindent \textbf{Inference.} For an unseen domain $\widehat D_T$, we first compute the domain embedding  $\bmu(\widehat D_T)$ using unsupervised samples (or the test set when none are present) from $\widehat D_T$ and the learned function $\Phi$ from (A). We then provide the computed embedding as an additional input to the model $F$ from (B) to get the \textit{adaptive} model $F(\cdot, \bmu(\widehat D_T))$.

\subsection{Prototypical Domain Embeddings}
\label{sec:proto}
\noindent \textbf{Kernel Mean Embeddings.} An elegant approach to embed distributions into vector spaces is the nonparametric method of \textit{kernel mean embeddings} (KME)~\cite{muandet2016kernel}. The fundamental idea behind KMEs is to consider each probability distribution $D \in \cD$ as a member of a reproducing kernel Hilbert space (RKHS) $\cH(\cX)$ with some kernel $k$ and feature $\Phi_\cD : \bbR^d \rightarrow \bbR^{d_D}$. Then, the KME of any distribution $D \in \cD$  can be given as $\bmu(D) = \int_{\cX} \Phi_\cD(\x)dD(\x)$.

The strength of this approach is that, if a suitable $\Phi_\cD$ is chosen, then it is possible to learn a function over the space of distributions $\cD$ via the average feature embedding of samples drawn from $D$. Hence, we propose to learn both a domain embedding $\bmu$ using features $\Phi_\cD$ and the target function $F(\cdot, \bmu(\cdot))$, such that when any domain $D$ is presented, we can provide an \textit{adapted} function $F(\cdot, {\bmu}(D))$. We learn a neural network $\Phi_\cD :\cX \rightarrow \bbR^{d_D}$, parameterized by weights $\btheta$, and then compute the $n$-sample \textit{empirical} KME $\bmu$ of each training domain $\widehat D$ as:
\begin{equation}
    \bmu(\widehat D) = \frac{1}{n}\sum_{\x \in \widehat D} \Phi_\cD(\x ; \btheta).
\label{eq:emb}
\end{equation}
\textbf{Low-Shot Domain Prototypes}.
\label{sec:prototype}
To be useful in classification, we require $\bmu$ to satisfy two central criteria:\vspace*{0.05in}\\
\noindent (A) \textbf{Expressivity.} $\bmu$, and consequently, $\Phi_\cD$, must be expressive, i.e., for two domains $D$ and $D'$, $\lVert \bmu(D) - \bmu(D') \rVert_2$ must be large if they are very distinct (i.e., domain shift), and small if they are similar. For example, if $\Phi_\cD$ is constant, the KME will provide no additional information about $D$, and we can expect the \textit{expected risk} itself to be large in this case, equivalent to non-adaptive classification.\vspace*{0.05in}\\
\noindent (B) \textbf{Consistency.} $\Phi$ must have little dependence on both the choice, and the number of samples used to construct it. This is desirable as we only have access to the empirical distribution $\widehat{D}$ corresponding to any domain $D$. Consistency can be achieved, if we learn $\Phi_\cD$ such that for any sample $\x$ from domain $D$, $\Phi_\cD(\x)$ is strongly clustered around $\bmu(D)$.

We show that \textit{prototypical networks}~\cite{snell2017prototypical} originally proposed for the task of few-shot learning, can be adapted to construct domain embeddings which satisfy both criteria. We train the network based on the algorithm from Snell~\etal~\cite{snell2017prototypical}, but using domain identities as the labels instead of class labels. The network is trained such that embeddings of all points from a domain are tightly clustered and far away from embeddings corresponding to points of other domains. In Section~\ref{sec:theory}, we demonstrate consistency by deriving a concentration bound for the clustering induced by the network. Expressivity is guaranteed by the loss function used to train the network, which forces domain embeddings to be discriminative, as we show next.


Our prototypical network (embedding function $\Phi_\cD$) is trained with SGD~\cite{bottou2010large}. At each iteration, we first sample a subset of $N_t$ domains from the training domains. Next, we sample two sets of points from each domain. The first set is used to obtain the embedding for the domain, by running the points through the network $\Phi_\cD$ and averaging (Eq.~\ref{eq:emb}). This results in an approximate embedding $\bmu(\widehat D)$ for each domain $D$. Every point $\x$ in the second set is given a probability of belonging to a domain $D$:
\begin{align}
    p_\btheta(\x \in D) = \frac{\exp\left(-\lVert \bmu(\widehat D) - \Phi_\cD(\x)\rVert_2^2 \right)}{\sum_{i=1}^{N_t} \exp\left(-\lVert \bmu(\widehat{D}_i) - \Phi_\cD(\x)\rVert_2^2\right)}.
    \label{eqn:proto_loss}
\end{align}
Learning proceeds by minimizing the negative log-likelihood $J(\btheta) = -\log p_\btheta(\x \in D_\x)$ for the correct domain $D_\x$ from which $\x$ was sampled. This ensures that probability is high only if a point is closer to its own domain embedding and farther away from other domains, making the domain embedding \textit{expressive}.  Once training is complete, we obtain the final model parameters $\btheta^*$. For each domain $\widehat{D}_i(n)$, the \textit{domain prototype} is constructed as:
\begin{equation}
    \bmu(\widehat{D}_i(n)) := \frac{1}{n}\sum_{j=1}^{n} \Phi_\cD(\x; \btheta^*). \label{eqn:final_prototype}
\end{equation}
Prototype training and construction is completely unsupervised as long as we are provided the inputs partitioned by domains. We repeat \textit{only} Equation~\ref{eqn:final_prototype} for each test domain as well (no retraining). See Algorithm~\ref{algo:algo1} for details.

\begin{algorithm}[t!]
\footnotesize
\caption{Prototypical Training Pseudocode}
\label{alg:proto_algo}
\underline{\textsc{Prototype Training}}\\
\textbf{Input.} $N$ domains $(D_i)_{i=1}^N$ with $n$ samples each. $N_t$ : \#domains sampled each batch, $N_s$: \#support examples, $N_q$: \#query examples per batch.\\
\textbf{Output.} Embedding neural network $\Phi_\cD : \cX \rightarrow \bbR^{d_D}$.

\begin{algorithmic}
\STATE \textbf{Initialize.}  weights $\btheta$ of $\Phi_\cD$ randomly.
\FOR{Round $t = 1$ to $T$}
\STATE $D_t \leftarrow \textsc{Randomly Sample}(N, N_t)$.  {\color{gray}// sample $N_t$ domains}
\FOR{Domain $d$ in $D_t$}
\STATE $\cS_d \leftarrow \textsc{Randomly Sample}(D_d, N_s)$. {\color{gray}// sample support points}
\STATE $\cS_q \leftarrow \textsc{Randomly Sample}(D_d, N_q)$. {\color{gray}// sample query points}
\STATE $\widehat\bmu_d \leftarrow \frac{1}{N_s}\sum_{\x \in \cS_d} \Phi_\cD(\x; \btheta)$. {\color{gray}// compute prototype}
\ENDFOR
\STATE $J_\btheta(t) \leftarrow 0$.
\FOR{Domain $d$ in $D_t$}
\FOR{$\x \in \cS_q$}
\STATE Update Loss $J_\btheta(t)$ based on Equation~\ref{eqn:proto_loss}.
\ENDFOR
\ENDFOR
\STATE $\btheta \leftarrow \textsc{SGD Step}(J(t), \btheta)$. {\color{gray}// gradient descent step}
\ENDFOR
\end{algorithmic}
\hrulefill\\
\underline{\textsc{Prototype Computation $(\Phi_\cD, \cS)$}}\\
\textbf{Input.} Trained prototype network $\Phi_\cD$ with weights $\btheta^*$.\\ 
Set $\cS$ of points to create prototype.\\ \textbf{Output.} Prototype $\bmu(\cS)$.
\begin{algorithmic}
\RETURN $\bmu(\cS) \leftarrow \frac{1}{|\cS|}\sum_{\x \in \cS} \Phi_\cD(\x, \btheta^*)$.
\end{algorithmic}
\label{algo:algo1}
\end{algorithm}

\subsection{ERM on Augmented Inputs}
\label{sec:augtrain}
Once we have the domain prototype for each training domain from the previous step, we create the \textit{augmented} training set by appending the corresponding prototype to each input. We obtain the training set $\widetilde\cS_\trn = \cup_{i=1}^{N} \widetilde{D}_i$, where, 
\begin{equation}
\widetilde{D}_i = \left(\x_{ij}, \bmu(\widehat{D}_i), y_{ij}\right)_{j=1}^{n}.
\label{eqn:augmented_dataset}
\end{equation}

The second step is to then train the \textit{domain-adaptive} classifier $F$ over $\widetilde\cS_\trn$ using regular ERM training. $F$ consists of two neural networks (Algorithm~\ref{algo:algo2}). The first network $F_{\text{ft}}$ with weights $\bomega_{\text{ft}}$ takes $\x$ (image) as input and outputs a feature vector $\Phi_\cX(\x)$. We then concatenate the image feature $\Phi_\cX$ with the domain embedding $\bmu(D)$ and pass through the second network $F_{\text{mlp}}$ which predicts the class label $y$. We refer to this approach as \daerm. In Section~\ref{sec:smallscale}, we present a few variants that combine the complimentary nature of adaptive and invariant approaches.
 
When testing, we first compute the domain prototypes for a test domain $\widehat D_T$ (Equation~\ref{eqn:final_prototype}) using a small number of unlabeled examples\footnote{We can also use a small sample of the total test set to create $\bmu(\widehat D_T)$.}, and then produce predictions via $F(\cdot, \widehat D_T)$. Alternatively, it is possible to decouple the prototype construction and inference procedures in the test phase (since we can use unsupervised samples obtained \textit{a priori} to construct prototypes). Note that once the prototype $\bmu(\widehat D_T)$ is computed, we only need the adapted classifier $F (\cdot, \bmu(\widehat D_T)$ for inference, not the prototypical network.

\begin{algorithm}[t]
\footnotesize
\caption{Adaptive Training Pseudocode}
\label{alg:main_algo}
\underline{\textsc{Adaptive Training}}\\
\textbf{Input.} Prototype network $\Phi_\cD$ with weights $\btheta^*$, Domains $(D_i)_{i=1}^N$. \\
\# points $N_p$ to sample from each domain to create prototype.\\
\textbf{Output.} $F_{\text{ft}}$ with weights $\bomega_{\text{ft}}$, $F_{\text{mlp}}$ with weights $\bomega_{\text{mlp}}$.

\begin{algorithmic}
\STATE {\color{gray} // compute prototypes for training domains}
\FOR{Domain $d$ from $1$ to $N$}
\STATE $\cS_d \leftarrow \textsc{Randomly Sample}(D_d, N_p)$.
\STATE $\bmu(\widehat D_d) \leftarrow$ \textsc{Prototype Computation}$(\Phi_\cD, \cS_d)$.
\ENDFOR
\STATE Create augmented dataset $(\widetilde D_i)_{i=1}^n$ (Equation~\ref{eqn:augmented_dataset}).
\FOR{SGD round $t=1$ to $T$}
\STATE Sample batch $(\x, \bmu, y)$ from augmented dataset.
\STATE $\Phi_\cX(\x) \leftarrow F_{\text{ft}}(\x; \bomega_{\text{ft}})$. {\color{gray}// compute image features}
\STATE $\Phi(\x, \bmu) \leftarrow \textsc{Concat}(\Phi_\cX(\x), \bmu)$. {\color{gray}// concatenate features}
\STATE $\hat y \leftarrow F_{\text{mlp}}(\Phi(\x, \bmu); \bomega_{\text{mlp}})$. {\color{gray}// compute predictions}
\STATE $J_{\bomega}(t) \leftarrow \textsc{CrossEntropy}(\hat y, y)$. {\color{gray}// compute loss}
\STATE $\bomega_{\text{ft}}, \bomega_{\text{mlp}} \leftarrow \textsc{SGD Step}(J_\bomega(t), \bomega_{\text{ft}}, \bomega_{\text{mlp}})$. {\color{gray}// gradient descent}
\ENDFOR
\end{algorithmic}
\hrulefill\\
\underline{\textsc{Adaptive Inference}}\\
\textbf{Input.} Trained networks $\Phi_\cD, F_{\text{ft}}, F_{\text{mlp}}$.\\ 
A set $\cS$ of $N$ points for prototype from domain $D$.\\
\textbf{Output.} Adaptive classifier for domain $D$.

\begin{algorithmic}
\STATE $\bmu \leftarrow$ \textsc{Prototype Computation}$(\Phi_\cD, \cS)$.
\RETURN $F(\x) = F_\text{mlp}\left(\textsc{Concat}\left(F_{\text{ft}}(\x), \bmu\right)\right)$.
\end{algorithmic}
\label{algo:algo2}
\end{algorithm}
\subsection{Theoretical Guarantees}
\label{sec:theory}
Here we provide an abridged summary of our theoretical contributions, and defer the complete results and proofs to the appendix. We first demonstrate that expressivity and consistency are theoretically motivated using the framework of kernel mean embeddings~\cite{muandet2016kernel}. Our first result is a concentration bound on the $n$-point approximation error of $\bmu$.
\begin{theorem}($\bmu$ approximation, informal)
For any domain $D$ and feature $\Phi$  such that $\bbE_{\x \sim D}[\lVert\Phi(\x)\rVert_2^2] \leq \sigma^2$, let the corresponding empirical dataset be $\widehat D(n)$, let $\bmu(D)$ be the true mean embedding and $\bmu(\widehat D)$ be its approximation. Then with high probability, $\lVert \bmu(D) - \bmu(\widehat D) \rVert_\infty \lesssim \widetilde{\mathcal O}(\sigma/\sqrt{n})$\footnote{The $\widetilde\cO$ notation hides polylogarithmic factors and failure probability.}. 
\label{thm:prototype}
\end{theorem}
This suggests that for low error, the variance $\sigma$ must be small for each domain $D \in \cD$, a property achieved by prototypes~\cite{snell2017prototypical}. Next, we provide a generalization bound that controls the gap between the training error $\widehat L_{n, N}(F)$ and test error $L(F)$ (RHS of Eqn.~\ref{eq:erm}) for $F$ that is \textit{adaptive}, i.e., a function of $\cX \times \cD$ and lies in the Hilbert space determined by some kernel $\kappa$. We build on the framework of~\cite{blanchard2011generalizing} and extend their results to a more general class of kernel functions.
\begin{theorem}(error bound, informal) Let $F$ be an adaptive classifier defined over $\cX \times \cD$ that lies in an RKHS $\cH_\kappa$ such that $\lVert f \rVert_{\cH_\kappa} \leq R$. Then, we have, with high probability that $|L(F) - \widehat L_{n, N}(F)| \lesssim \widetilde{\mathcal O}\left(R\cdot \sigma((\log N) /n)^{\frac{1}{2}} + RN^{-\frac{1}{2}}\right)$.
\label{thm:risk_bound}
\end{theorem}
Theorem~\ref{thm:risk_bound} quantifies the error in terms of the feature variance $\sigma$, ``norm'' of function $R$ and the number of samples $n$. In contrast to the results presented in prior work~\cite{blanchard2011generalizing, muandet2013domain, hu2020domain}, we provide a tighter dependence on $n$ that incorporates variance $\sigma$ as well. See Appendix for more details.
\section{Large Scale Benchmarks}
\label{sec:benchmarks}
There has been significant interest recently in domain generalization for vision problems \cite{gulrajani2020search}. However, we still lack large-scale benchmarks for rigorous evaluation. The largest dataset currently in use consists of only $7$ domains and $5$ classes~\cite{peng2019moment}, much smaller compared to real-world settings. Moreover, the model selection method of choice for this benchmark is \textit{leave-one-domain-out} cross validation, which is infeasible in most large-scale applications with many domains. To address these issues, we introduce two large-scale domain generalization benchmarks.
\subsection{Geo-YFCC}
YFCC100M~\cite{thomee2016yfcc100m} is a large-scale dataset consisting of 100M images sourced from Flickr, along with tags submitted by users. Geotags (latitude and longitude) are also available for many of these images. We construct a dataset where each domain corresponds to a specific country.\vspace*{0.05in}\\
\noindent{\textbf{Construction.}} We use geotags to partition images based on their country of origin. For the label space $\cY$, we consider the 4K categories from ImageNet-5K~\cite{deng2009imagenet} not present in ILSVRC12~\cite{russakovsky2015imagenet}. These categories are selected in order to eliminate biased prior knowledge from pre-training on ILSVRC12, which is the de-facto choice for large-scale pre-training. For each of the 4K labels, we select the corresponding images from YFCC100M based on a simple keyword-filtering of image tags. This provides us $1,261$ categories with at least 1 image present. Furthermore, each category is present in at least $5$ countries. We group images by their country of origin and only retain countries that have at least $10$K images. For any domain with more than $20$K images, we randomly sub-sample to limit it to $20$K images. Therefore, each \textit{domain} (i.e., country) has anywhere between $10$K-$20$K images, giving us a total of 1,147,059 images from 1,261 categories across $62$ countries (domains), and each image is associated with a class label and country (domain). We refer to this resulting dataset as {\textbf{Geo-YFCC}}. To the best of our knowledge, the scale of this dataset in the number of images, labels and domains is orders of magnitude more than prior datasets.\vspace*{0.05in}\\
\noindent{\textbf{Train-Test Split.}} We randomly partition the data in to $45$ training, $7$ validation and $15$ test domains (by country). For each domain, we sample $3$K points to create a per-domain test set and use the remaining points for training and validation. Since any image may have multiple labels, we convert it to a single-label dataset by replicating each such image and associating each separate copy for each label\footnote{This does not create any overlap between the train and test data as this duplication is done after the train-test split. As each image can belong only to one country, this ensures that no images are shared across domains.}, expanding the total image set to 1,809,832. We plan on releasing the complete partitions and annotations upon publication.\vspace*{0.05in}\\
\noindent{\textbf{Analysis: }}Geo-YFCC exhibits several properties of real-world data, such as long-tailed label distributions~\cite{liu2019large}, considerable \textit{covariate shift}~\cite{sugiyama2007covariate} and \textit{label shift}~\cite{azizzadenesheli2019regularized} across domains. Figure~\ref{fig:visualizations}A shows the label distributions from $4$ sample domains, where the categories are ordered based on the frequency of labels from the ``USA'' domain. We see that each domain exhibits long tails and labels shifts (i.e., the marginal distribution of labels is different in each domain~\cite{lipton2018detecting}). For example, {{``Cambodia"}} has  {{``Temple"}} as the most frequent category while it is a low-frequency concept for {{``USA"}}. Figure~\ref{fig:visualizations}B shows sample images for some categories from different domains, highlighting substantial \textit{covariate} shift (i.e., shift in the marginal distribution of $\cX$).

\begin{figure*}[t!]
\centering
\includegraphics[width=0.9\linewidth]{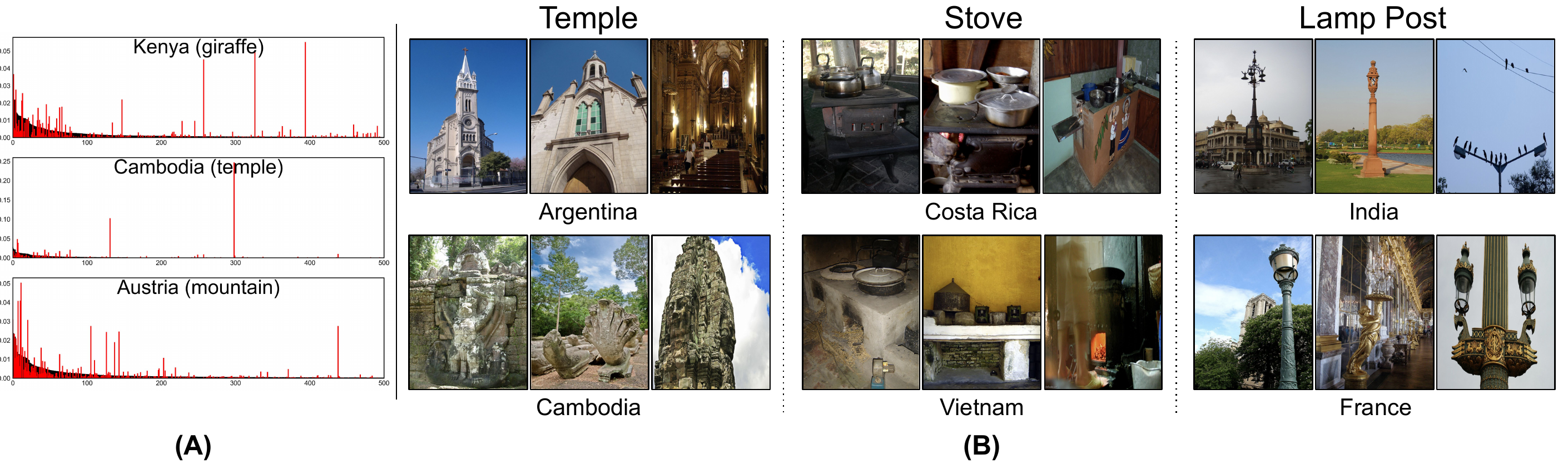}
\caption{(A, left) Relative frequencies of categories from 3 different domains overlaid on the those of the USA, with most frequent category in parentheses. (B, right) Samples from 3 classes from different domains, highlighting covariate shift.}
\label{fig:visualizations}
\end{figure*}

\subsection{Long-Tailed ImageNet (LT-ImageNet)}
We construct another dataset from the ImageNet-5K~\cite{deng2009imagenet} benchmark by artificially introducing label shifts and long-tailed label distributions. The controlled settings in this dataset allow us to perform various ablations and understand our algorithm (Section~\ref{sec:exp_ablations}). We subsample 500 categories from the 5K categories as the \textit{base set} (excluding ILSVRC classes as earlier), from which we partition samples from each class into distinct train, val and test subsets. Each training domain is constructed as follows. We first select $K$ head classes randomly from the base set. We then randomly sample $A$ points from train samples of each head class and $Af$, $f<1$ points from the remaining $500-K$ tail classes to construct the train split of that domain, in order to mimic long-tailed phenomena. We use $A=50$ and $300$ to construct val and test splits for the domain from the corresponding splits of each class. We repeat this process $N$ times to construct $N$ training domains. Our overall training set has a non-uniform representation of classes in each domain, along with variation in long-tailed effects. We refer to the resulting dataset as {\textbf{\ltin}}.

To construct the val and test set of domains, we sample $10$ domains with classes restricted to the set of all classes present in the training set (since we are not tackling zero-shot domain generalization), to ensure that the training data has at least 1 sample from each class (across all domains). 
We believe that these two large-scale datasets will help the community in properly understanding the effectiveness of existing and future domain generalization algorithms.
\section{Experiments}
We compare our algorithm with various recent algorithms on domain generalization problems both at the small-scale (i.e., fewer domains and classes per domain, and significant hyper-parameter tuning is possible) and at the real-world (large-scale) level. All benchmark algorithms are implemented from the \textsc{DomainBed}~\cite{gulrajani2020search} suite with the hyper-parameter value ranges suggested therein.

\noindent{\textbf{Domain Embeddings (Section~\ref{sec:proto}).} }We use a ResNet-50~\cite{he2016deep} neural network pre-trained on ILSVRC12~\cite{russakovsky2015imagenet} and introduce an  additional fully-connected layer of dimension  $d_D$ on top of the \texttt{pool5} layer. The output of this layer is used to construct the domain embedding (Equation~\ref{eqn:final_prototype}). Unless otherwise stated, $d_D=1024$, and, the number of domains $N_t$ used in $N_t$-way classification during each training step is set to $4$. See Appendix for optimal hyper-parameters. To construct the domain embedding (Equation~\ref{eqn:final_prototype}), we use at most $200$ random points from the test set itself for small scale benchmarks in Section~\ref{sec:smallscale}. For large-scale benchmarks, we leverage a subset of the samples from training domains and use a separate set of held-out points for test domains. For LT-ImageNet and Geo-YFCC, $300$ and 5K points are used respectively.\vspace*{0.05in}\\
\noindent{\textbf{DA-ERM (Section~\ref{sec:augtrain}).} }
The first neural network, $F_{\text{ft}}$, is again a ResNet-50~\cite{he2016deep} network pre-trained on ILSVRC12~\cite{russakovsky2015imagenet} with an additional fully-connected layer of dimension $1024$ added to the \texttt{pool5} layer. This layer is then concatenated with the
$d_D$ dimensional domain embedding and passed through $F_{\text{mlp}}$. Unless otherwise stated, the MLP has two layers  with hidden dimension $d_{\text{mlp}}=1024$, and output dimension set to number of classes. We use the standard cross-entropy loss for training. See the Appendix for optimal training hyper-parameters for each dataset.

\begin{table*}[th!]
\begin{minipage}{.6\textwidth}
\centering
\footnotesize
\rowcolors{3}{gray!6}{white}
\begin{tabular}{ccccccc}
\hline
\hline
\textbf{Algorithm} & VLCS~\cite{fang2013unbiased}        & PACS~\cite{li2017deeper}                 & OffHome~\cite{venkateswara2017Deep}          & DNet~\cite{peng2019moment}           &  TerraInc~\cite{beery2018recognition}       & Average               \\
\hline
ERM~\cite{gulrajani2020search}      &76.7  $\pm$   0.9         &83.2  $\pm$  0.7        &67.2 $\pm$  0.5         &41.1 $\pm$  0.8         &
46.2  $\pm$  0.3         &62.9  $\pm$  0.6         \\
IRM~\cite{arjovsky2019invariant}  &76.9  $\pm$  0.5         &82.8  $\pm$  0.5        &67.0 $\pm$  0.4         &35.7 $\pm$  1.9         & 
43.8  $\pm$  1.0         &61.2  $\pm$  0.8         \\
DRO~\cite{sagawa2019distributionally}     &77.3  $\pm$  0.2         &83.3 $\pm$  0.3         &66.8 $\pm$  0.1         &33.0 $\pm$  0.5         &
42.0  $\pm$  0.6         &60.4  $\pm$  0.4         \\
Mixup~\cite{wang2020heterogeneous}    &78.2  $\pm$  0.6         &83.9 $\pm$  0.8         &68.3 $\pm$  0.4         &40.2 $\pm$  0.4         &
46.2  $\pm$  0.5         &63.3  $\pm$  0.7         \\
MLDG~\cite{li2017learning}     &76.8  $\pm$  0.4         &82.6 $\pm$  0.6         &67.7 $\pm$  0.3         &42.2 $\pm$  0.6         & 
46.0  $\pm$  0.4         &63.0  $\pm$  0.4         \\
CORAL~\cite{sun2016deep}    &77.3  $\pm$  0.3         &83.3 $\pm$  0.5         &68.6 $\pm$  0.1         &42.1 $\pm$  0.7         &
47.7  $\pm$  0.3         &63.8  $\pm$  0.3         \\
MMD~\cite{li2018domain}      &76.1  $\pm$  0.7         &83.1 $\pm$  0.7         &67.3 $\pm$  0.2         &38.7 $\pm$  0.5         &
45.8  $\pm$  0.6         &62.2  $\pm$  0.6         \\
C-DANN~\cite{li2018deep}   &73.8  $\pm$  1.1         &81.4 $\pm$  1.3         &64.2 $\pm$  0.5         &39.5 $\pm$  0.2         &
40.9  $\pm$  0.7         &60.1  $\pm$  0.9         \\ 
\hline
DA-ERM   &78.0  $\pm$  0.2         &84.1 $\pm$  0.5         &67.9 $\pm$  0.4         &43.6 $\pm$  0.3         &
47.3  $\pm$  0.5         &64.1  $\pm$  0.8         \\
DA-MMD   &\textbf{78.6  $\pm$  0.5}&84.4 $\pm$  0.3         &68.2 $\pm$  0.1         &43.3 $\pm$  0.5         &
47.9  $\pm$  0.4         &64.5  $\pm$  0.5         \\
DA-CORAL &78.5  $\pm$  0.4         &\textbf{84.5 $\pm$  0.7}&\textbf{68.9 $\pm$  0.4}&\textbf{43.9 $\pm$  0.3}&
\textbf{48.1  $\pm$  0.3}&\textbf{64.7  $\pm$  0.6}\\ 
\hline
\hline
\end{tabular}
\caption{Small-Scale Benchmark Comparisons on \textsc{DomainBed}.}
\label{tab:exp_small_scale}
\end{minipage}
\begin{minipage}{.3\textwidth}
\rowcolors{4}{gray!6}{white}
\footnotesize
\centering
\captionsetup{justification=centering}
\begin{tabular}{c|cc|cc}
\hline
\hline
\multirow{3}{*}{\textbf{Algorithm}} & \multicolumn{2}{c|}{LT-ImageNet} & \multicolumn{2}{c}{Geo-YFCC} \\
         & \multicolumn{1}{c}{Train}              &  \multicolumn{1}{c|}{Test}              & \multicolumn{1}{c}{Train}                & \multicolumn{1}{c}{Test}                   \\
         & Top-1 / 5 & Top-1 / 5 & Top-1 / 5 & Top-1 / 5\\
\hline
ERM~\cite{gulrajani2020search}      & 70.9 / 90.3 & 53.7 / 77.5 & \textbf{28.5} / \textbf{56.3} & 22.4 / 48.2      \\
CORAL~\cite{sun2016deep}    & 71.1 / 89.3 & 53.5 /  76.2 & 25.5 / 51.2 & 21.8 / 46.4          \\
MMD~\cite{li2018domain}      & 70.9 / 88.0 & 52.8 / 76.4 & 25.4 / 50.9 & 21.8 / 46.2        \\
\hline
DA-ERM   & \textbf{73.4} / \textbf{91.8} & \textbf{56.1} / \textbf{80.5} & 28.3 / 55.8 & \textbf{23.4} / \textbf{49.1} \\
\hline
\hline
\end{tabular}
\caption{Large-Scale Comparisons.}
\label{tab:exp_large_scale}
\end{minipage}
\end{table*}

\subsection{Small-Scale Benchmarks} 
\label{sec:smallscale}
We conduct small-scale benchmark comparisons carefully following the \textsc{DomainBed} suite~\cite{gulrajani2020search}, which contains 7 datasets in total. This paper focuses specifically on real-world domain generalization, and hence we select the 5 larger datasets (VLCS~\cite{fang2013unbiased}, PACS~\cite{li2017deeper}, Office-Home~\cite{venkateswara2017Deep}, Domain-Net~\cite{peng2019moment} and Terra Incognita~\cite{beery2018recognition}) for our experiments, and discard the MNIST-based~\cite{lecun1998mnist} datasets.

In this setting we select hyper-parameters by \textit{leave one domain out} cross validation, the de facto technique for model selection for domain generalization, wherein we run $N$ trials of each hyper-parameter setting, setting one domain for testing and the rest for training, and selecting the set of hyper-parameters that maximize validation accuracy over the training domains, averaged over all $N$ trials. Performance under alternative techniques for model selection as well as the training details and hyper-parameter ranges for different algorithms can be found in the Appendix.\vspace*{0.05in}\\
\noindent\textbf{Domain Mixup.} Since the small-scale datasets have typically 4 and at most 7 domains in total (1 of which is used as testing), the resulting training set (of domains) is not representative enough to learn a sufficiently discriminatory domain prototype. To alleviate this issue, we introduce \textit{mixup}~\cite{zhang2017mixup} in the prototype training step, where, for each batch of points from training domains, we construct \textit{synthetic} domains by averaging points belonging to different domains and providing that as additional samples (belonging to a new domain). For example, during training, if we have a batch of points from each of the $N_t$ training domains, we will create additional $N_t$ synthetic batches where each synthetic batch is created by randomly averaging two batches. The averaging ratio for each synthetic domain is randomly chosen uniformly between $(0.2, 0.8)$.\vspace*{0.05in}\\
\noindent\textbf{Results.} Results are summarized in Table~\ref{tab:exp_small_scale}, where each experiment is averaged over 8 independent trials (for our algorithm, this includes retraining of the domain prototypical network). \daerm  provides an average improvement of \textbf{1.2\%} in accuracy across all 5 datasets over the ERM baseline that simply combines data from all the training domains to train a model. The improvement is more substantial in Domain-Net~\cite{peng2019moment}, that has a larger number of domains. This can be attributed to the increased diversity in the training data, which allows us to construct embeddings that cover the space of domains better. In comparison to other domain-invariant approaches, \daerm either performs better or competitively with the best approach on each individual dataset.

To demonstrate the complimentary value of \textit{adaptive} learning to domain-invariant methods, we modify the second step of our algorithm and consider two variants of the \daerm approach: \dammd and \dacoral, by replacing the cross-entropy in \daerm with the loss terms in domain-invariant approaches MMD~\cite{li2018domain} and CORAL~\cite{sun2016deep} respectively\footnote{Auxiliary loss terms minimizing the differences between the domains is applied to the $d_{\text{mlp}}$ dimensional hidden layer in the MLP $F_{\text{mlp}}$.} (see the appendix for details). Compared to CORAL, {\dacoral} provides a significant improvement of $0.9\%$, and \dammd achieves an impressive improvement of \textbf{2.3\%} ($62.2\%$ vs. $64.5\%$) over MMD.

\subsection{Large-Scale Benchmarks}
\label{sec:large_scale}
We now present results on two large-scale benchmarks proposed in Section~\ref{sec:benchmarks}. We compare our approach with ERM, CORAL~\cite{sun2016deep} (the best performing baseline in Section~\ref{sec:smallscale}), and, MMD~\cite{li2018domain}. We were unable to replicate several algorithms at scale, such as DANN~\cite{ganin2016domain}, Conditional-DANN~\cite{long2018conditional} and MMLD~\cite{pei2018multi}: these algorithms rely on adversarial training that requires careful hyper-parameter optimization, which we were unable to accomplish based on hyper-parameters suggested in \textsc{DomainBed}~\cite{gulrajani2020search}.
\vspace*{0.05in}\\
\noindent{\textbf{\ltin.} }We set $N$=50, $K$=100, $A$=350 and $f$=0.1 to construct the LT-ImageNet (50, 100, 350, 0.1) dataset and report top-1 and top-5 accuracies on the test splits of train and test domains in Table~\ref{tab:exp_large_scale} (left). We observe that \daerm improves top-1 performance by $2.3\%$ ($71.1\%$ vs. $73.4\%$) on train domains and $2.4\%$ ($53.7\%$ vs $56.1\%$) on test domains. 
We show the effect of varying the dataset settings ($N$, $K$, $A$ and $f$) in Section~\ref{sec:exp_ablations}.\vspace*{0.05in}\\
\noindent{\textbf{Geo-YFCC.} }
Table~\ref{tab:exp_large_scale} (right) shows comparisons on the test splits for train and test domains of the Geo-YFCC dataset. \daerm provides an absolute improvement of 1\% in comparison to benchmark techniques on test domains. It must be noted that this problem is substantially more challenging, as we have 1,261 classes with significant domain differences between the countries (Section~\ref{sec:benchmarks}). We also observe that all other domain-invariant baselines were unable to obtain improvements over the simple ERM baseline. Figure~\ref{fig:ablations}A shows t-SNE~\cite{maaten2008visualizing} visualizations of prototypical embeddings for each domain, grouped by continent. Countries belonging to the same continent are mapped closer together, expressing the semantics captured by $\bmu$. Additionally, we note that our algorithm simultaneously improves performance on test domains while maintaining performance on training domains, unlike other techniques.

\begin{figure*}[htb]
\centering
\includegraphics[width=0.85\linewidth]{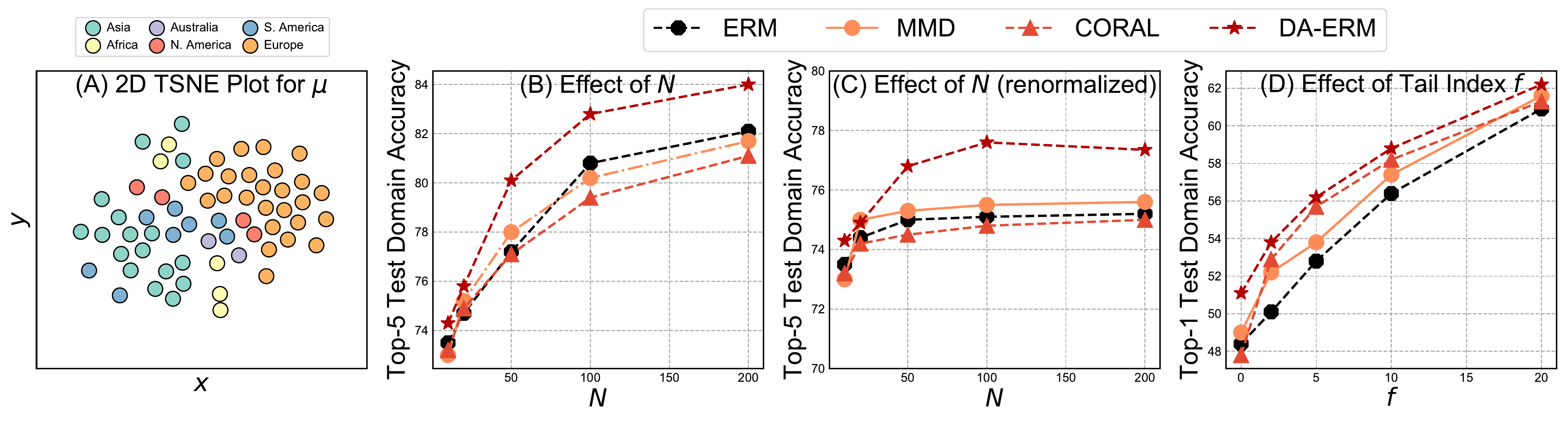}
\caption{Ablation experiments: Fig. A (left) is a t-SNE visualization of domain embeddings. Figs. B (center left) and C (center right) show the effect of changing the number of domains, where the total number of training points increases linearly with the number of domains in (B) and is constant in (C). Fig. D (right) demonstrates the effect of the tail index $f$.}
\label{fig:ablations}
\end{figure*}
\subsection{Ablations}
\label{sec:exp_ablations}
\subsubsection{Dataset Ablations}
We consider the LT-ImageNet dataset with default settings mentioned in Section~\ref{sec:large_scale}, unless otherwise stated.
\vspace*{0.05in}\\
\noindent \textbf{Effect of diversity in training data.} 
We first vary the number of available training domains ($N$), while sampling the same number of points from each domain. From Figure~\ref{fig:ablations}B, we observe that compared to other algorithms, the improvements in domain-aware training are larger, which we attribute to diversity in prototype training. Next, we repeat the same experiment, however, we subsample points from each domain such that the overall number of training points are the same, i.e., $NA$, remains constant. We see in Figure~\ref{fig:ablations}C that \daerm performs much better than other algorithms with increasing number of domains. Note that once $N > 100$, a slight dip is observed in our algorithm. This can be understood intuitively from Theorem~\ref{thm:risk_bound}, as it suggests such an inflection point when the total number of points ($nN$) is fixed due to the interaction of the two terms of the RHS. As the number of domains $N$ increases, the first term decreases, while the second term increases.
\vspace*{0.05in}\\
\noindent \textbf{Effect of Tail Index.} In Figure~\ref{fig:ablations}D, we demonstrate the effect of varying $f$: the parameter $f$ controls the extent of imbalance present in each domain, by adjusting the fraction of points available in the \textit{tail classes} compared to the head classes. We see that when the dataset is severely imbalanced (i.e., $f < 5\%$), domain-adaptive training considerably outperforms existing techniques, with a top-1 performance gap of about \textbf{3.5\%} on average comapred to naive ERM. As $f$ increases and the imbalance becomes smaller, our method performs closer to ERM itself, with an average improvement of only \textbf{1.8\%}. Since most real-world datasets are extremely long-tailed, we expect this to be a valuable feature of domain-adaptive learning in practice.
\subsubsection{Algorithm Ablations}
We perform several ablation studies on the domain embedding algorithm itself to verify the two properties of \textit{consistency} and \textit{expressivity} defined in Section~\ref{sec:proto}.
\vspace*{0.05in}\\
\noindent \textbf{Consistency.} We varied the number of points used to construct the prototype on Geo-YFCC and observed that the performance declines only for fewer than $50$ points. For other values up to $2$K, performance remains virtually identical (see Appendix).  
This is desirable as in many settings we do not have access to many samples from new domains.\vspace*{0.05in}\\
\noindent \textbf{Expressivity.} We measure the ability of the prototype to efficiently embed different domains distinctly within the embedding space. We evaluate the following two alternate approaches to construct domain embedding: (a) {\textit{Mean}}: we select the ERM baseline model and simply compute the average \texttt{pool5} features for each domain; (b) {\textit{Softmax}}: we train the domain embedding network via cross-entropy with domain IDs as labels instead of the prototypical loss. The results are summarized in Table~\ref{tab:exp_abl_proto_training_type}. For reference, we also report results of naive ERM ({\textit{None}}) and supplying incorrect embeddings ( i.e., the embedding of some other random domain) during evaluation of test domains ({\textit{Random}}). We observe that using random features degrades performance significantly, indicating that domain embeddings indeed play an important role in the final model. Mean features, while useful (i.e., perform better than naive ERM), are not as expressive as prototypical learning. Softmax features show mixed results by performing worse than Prototype on {\ltinp } and slightly better on Geo-YFCC. We additionally study the effect of number of domains sampled per round in prototypical training, and observe no significant effect for values in the range $[3, 8]$.
\begin{table}
    \centering
    \footnotesize
    \begin{tabular}{ccccccc}
    \hline
    \hline\multirow{2}{*}{Dataset} & \multicolumn{6}{c}{Top-1 Accuracy on Test Domains} \\
                  & None & Random   & Mean    & Softmax   & Prototype\\
    \hline
    LT - ImageNet & 76.9 & 75.3     & 77.9 & 78.8     & 80.5   &       \\
    Geo - YFCC    & 22.4    & 21.3 & 23.0       & 23.5      & 23.4  \\
    \hline
    \hline
    \end{tabular}
    \caption{Ablation of various domain embedding algorithms.}
    \label{tab:exp_abl_proto_training_type}
\end{table}
\section{Conclusion}
The problem of domain generalization lies at the heart of many machine learning applications. Research on the problem has largely been restricted to small-scale datasets that are constructed from artificially diverse sources that are not representative of real-world settings. For example, it is impractical to train a model on sketches if the target domain contains photographs: it is much easier in practice to curate a dataset of photographs instead. With this paper \textit{work}, our goal is to provide an algorithm and evaluation procedure that is closer to real-world applications, and involves experimentation at scale. We presented an \textit{adaptive} domain generalization framework, and conducted the first set of experiments at scale for this problem. Our results suggest that we need significant effort to scale to real-world problem settings, and with our contributions we expect to initiate research in this new direction.
\\
\noindent\textbf{Acknowledgements}. We would like to sincerely thank Deepti Ghadiyaram, Filip Radenovic and Cheng-Yang Fu for their helpful feedback on our paper.

\setlength{\tabcolsep}{2pt}
\appendix
\onecolumn
\section{Theory}
\subsection{Motivation}
The traditional objective of the domain generalization problem is to learn a function $f^* : \cX \rightarrow \cY$ that minimizes the empirical risk over $\fP$, i.e., for some class of functions $\cF$ and loss function $\ell : \bbR \times \cY \rightarrow \bbR_+$,
\begin{equation}
    f^* = \argmin_{f \in \cF} \bbE_{D \sim \fP}\left[\bbE_{(\x, y )\sim D}\left[\ell(f(\x), y)\right]\right].
\end{equation}
We denote the RHS in the above equation for any $f \in \cF$ as the \textit{expected risk} $L(f)$. Correspondingly, the optimal ERM solution $\hat f$ on the training data can be given as follows.
\begin{equation}
    \hat f = \argmin_{f \in \cF} \frac{1}{nN}\left(\sum_{\widehat D \in \cS_\trn}\sum_{(\x, y )\in D}\ell(f(\x), y)\right).
\end{equation}
We denote the RHS in the above equation for any $f \in \cF$ as the \textit{empirical risk} $\widehat L(f)$. Existing \textit{invariant} approaches~\cite{ben2010theory} build on exploiting the traditional decomposition of the empirical risk as a function of the variance of $f$ across $\fP$.
\begin{equation}
   \underset{\text{generalization gap}}{\left|L(f)-\widehat L(f)\right|} \preccurlyeq \underbrace{\mathcal O\left(\frac{\text{Var}_{D \in \cS_\trn}(f)}{N}\right)^{\frac{1}{2}}}_{\text{inter-domain variance}}.
\end{equation}
The particular approximation of the variance penalty (term 2 of the RHS) leads to the class of \textit{invariant} domain generalization algorithms. In contrast, in this paper we focus on controlling the LHS directly by selecting a stronger function class $\cF$ over the product space $\cD \times \cX$ (instead of $\cX$)\footnote{While we do not investigate this in detail, our approach is compatible with \textit{invariant} approaches, as we consider separate aspects of the problem.}. The key modeling choice we make is to learn a function $F :  \cD \times \cX \rightarrow \cY$ that predicts $\hat{y} = F(D_X, \x)$ for any $(\x, y) \in D$, where $D_X$ is the marginal of $\cX$ under $D$.

\subsection{Background}
Let $\mathcal X \subset \mathbb R^d$ be a \textit{compact} input space, $\mathcal Y \in [-1, 1]$ to be the output space, and let $\mathfrak P_{\cX \times \cY}$ denote the set of all probability distributions over the measurable space $(\cX \times \cY, \Sigma)$, where $\Sigma$ is the $\sigma$-algebra of subsets of $\cX \times \cY$. Additionally, we assume there exists sets of probability distributions $\fP_\cX$ and $\fP_{\cY | \cX}$ such that for any sample $P_{XY} \in \fP_{\cX\times\cY}$ there exist samples $P_X \in \fP_\cX$ and $P_{Y|X} \in \fP_{\cY|\cX}$ such that $P_{XY} = P_X \bullet P_{Y|X}$ (this characterization is applicable under suitable assumptions, see Sec. 3 of~\cite{blanchard2011generalizing}). We assume that there exists a measure $\mu$ over $\fP_{\cX \times \cY}$ and each \textit{domain} $\cD$ is an i.i.d. sample from $\cP_{\cX \times \cY}$, according to $\mu$. The \textit{training set} is generated as follows. We sample $N$ realizations $P_{XY}^{(1)}, ..., P_{XY}^{(n)}$ from $\fP_{\cX\times\cY}$ according to $\mu$. For each domain, we then are provided $\cD_i$, which is a set of $n_i$ i.i.d. samples $(\x_j^{(i)}, y_j^{(i)})_{j \in [n_i]}$ sampled i.i.d. from $P_{XY}^{(i)}$.

Each \textit{test} domain is sampled similarly to the training domain, where we first sample a probability distribution $P^T_{XY} \sim \mu$, and are then provided $n_T$ samples $(\x_j^{T}, y_j^{T})_{j \in [n_T]}$ from $P^T_{XY}$ that forms a test domain. The key modeling assumption we make is to learn a decision function $f :  \fP_\cX \times \cX \rightarrow \bbR$ that predicts $\hat{y} = f(\widehat P_X, \x)$ for any $(\x, y) \in \cD$ and $\widehat P_X$ is the associated empirical distribution of $\cD$. For any loss function $\ell : \bbR \times \cY \rightarrow \bbR_+$, then the empirical loss on any domain $\cD_i$  is $\frac{1}{n_i}\sum_{(\x, y) \in \cD_i} \ell(f(\widehat{P}^{(i)}_X, \x), y)$. We can define the average generalization error over $N$ test domains with $n$ samples as,
\begin{equation}
    \widehat{L}_N(f, n) := \frac{1}{N}\sum_{i \in [N]}\left[\frac{1}{n}\sum_{(\x, y) \in \cD_i} \ell(f(\widehat{P}^{(i)}_X, \x), y)\right].
\end{equation}
In the limit as the number of available domains $N \rightarrow \infty$, we obtain the expected $n$-sample generalization error as,
\begin{equation}
    L(f, n) := \bbE_{P_{XY} \sim \mu, \cD \sim (P_{XY})\otimes n}\left[\frac{1}{n}\sum_{i=1}^n \ell(f(\widehat{P}_X, \x_i), y_i)\right].
\end{equation}
The modeling assumption of $f$ being a function of the empirical distribution $\widehat{P}_X$ introduces the primary difference between the typical notion of \textit{training error}. As $n \rightarrow \infty$, we see that the empirical approximation $\widehat{P}_X \rightarrow P_X$. This provides us with a true generalization error, in the limit of $n\rightarrow \infty$ (i.e., we also have complete knowledge of $P_X$),
\begin{equation}
    L(f, \infty) := \bbE_{P_{XY} \sim \mu, (\x, y) \sim P_{XY}}\left[\ell(f(P_X, \x), y)\right].
\end{equation}
An approach to this problem was proposed in Blanchard~\etal\cite{blanchard2011generalizing} that utilizes product kernels. The authors consider a kernel $\kappa$ over the product space $\fP_\cX \times \cX$ with associated RKHS $\cH_\kappa$. They then select the function $f_\lambda$ such that 
\begin{align}
    f_\lambda  = \argmin_{f \in \cH_\kappa} \frac{1}{N}\sum_{i=1}^N \frac{1}{n_i} \sum_{j=1}^{n_i} \ell(f(\widehat{P}^{(i)}_X, \x_{ij}), y_{ij}) + \lambda \lVert f \rVert_{\cH_\kappa}^2.
\end{align}
The kernel $\kappa$ is specified as a Lipschitz kernel on $\fP_\cX \times \cX$:
\begin{align}
    \kappa((P_X, \x), (P_{X'}, \x')) = f_\kappa(k_P(P_X, P_{X'}, k_X(\x, \x')).
    \label{eqn:product_kernels}
\end{align}
Where $K_P$ and $K_X$ are kernels defined over $\fP_\cX$ and $\cX$ respectively, and $f_\kappa$ is Lipschitz in both arguments, with constants $L_P$ and $L_X$ with respect to the first and second argument respectively. Moreover, $K_P$ is defined with the use of yet another kernel $\fK$ and feature extractor $\phi$:
\begin{align}
    k_P(P_X, P_{X'}) = \fK(\phi(P_X), \phi(P_{X'})).
\end{align}
Here, $\fK$ is a necessarily non-linear kernel and $\phi$ is the kernel mean embedding of $P_X$ in the RKHS of a distinct kernel $k'_X$, i.e.,
\begin{align}
    P_X \rightarrow \phi(P_X) := \int_\cX k'_X(\x, \cdot) dP_X(\x).
\end{align}
\subsection{Kernel Assumptions}

To obtain bounds on the generalization error, the kernels $k_X, k'_X$ and $\fK$ have the assumptions of boundedness, i.e., $k_X(\cdot, \cdot) \leq B^2_k$, $k'_X(\cdot, \cdot) \leq B^2_{k'}$ and $\fK(\cdot, \cdot) \leq B^2_\fK$. Moreover, $\phi$ satisfies a H\"older condition of order $\alpha \in (0, 1]$ with constant $L_\fK$ in the RKHS ball of $k'_X$, i.e., $\forall \x, \x' \in \cB(B_{k'}), \lVert \phi(\x) - \phi(\x') \rVert \leq L_\fK \lVert \x - \x' \rVert^\alpha$.
\subsection{Consistency Bound (Theorem 1 of the Main Paper)}
We first state the formal theorem.
\begin{theorem}
Let $D$ be a distribution over $\cX$, and $\widehat{D}$ be the empirical distribution formed by $n$ samples from $\cX$ drawn according to $D$, and $\lVert\Phi_\cD(\x)\rVert \leq B_{k'}, \bbE_{\x \sim D}[\lVert \Phi_\cD(\x) \rVert_2^2] \leq \sigma^2$. Then, we have with probability at least $1-\delta$ over the draw of the samples,
\begin{align*}
    \left\lVert \bmu(D) - \bmu(\widehat D)\right\rVert_\infty \leq \sqrt{\frac{8\sigma^2\log(1/\delta)}{n}- \frac{1}{4}}.
\end{align*}
\label{lem:feature_variance}
\end{theorem}
\begin{proof}
We first note that $\bmu(D) = \bbE_{\x \sim D}[\Phi(\x)]$, and each $\Phi(\x)$ is a $d_D-$dimensional random vector. Furthermore, since the kernel $k'_X$ is bounded, we have that $\lVert \Phi(\x) \rVert \leq B_{k'}^2$ for any $\x \in \cX$. Next, we present a vector-valued Bernstein inequality.
\begin{lemma}[Vector-valued Bernstein bound~\cite{}]
Let $\x_1, \x_2, ..., \x_n$ be $n$ finite-dimensional vectors such that $\bbE[\x_i] = 0$, $\lVert \x_i \rVert \leq \mu$ and $\bbE[\lVert \x_i \rVert^2] \leq \sigma^2$ for each $i \in [n]$. Then, for $0 < \epsilon < \sigma^2/\mu$, 
\begin{align*}
    \bbP\left(\left\lVert \frac{1}{n}\sum_{i=1}^n \x_i \right\rVert \geq \epsilon\right) \leq \exp\left(-n\cdot\frac{\epsilon^2}{8\sigma^2} - \frac{1}{4}\right).
\end{align*}
\label{lem:bernstein}
\end{lemma}
Now, we have, by the above result, for any $0 < \epsilon < \sigma^2/B_{k'}$
\begin{align}
    \bbP\left(\left\lVert \bmu(D) - \bmu(\widehat D)\right\rVert \geq \epsilon\right) \leq \exp\left(-n\cdot\frac{\epsilon^2}{8\sigma^2} - \frac{1}{4}\right).
\end{align}
This is obtained by noting that $\bmu(\widehat D) = \frac{1}{n}\sum_{i=1}^n \Phi(\x_i)$ and that $\bmu(D) = \bbE_{\x \sim D}[\Phi(\x)]$. Setting the RHS as $\delta$ and rearranging terms gives us the final result.
\end{proof}
\subsection{Generalization Bound (Theorem 2 of Main Paper)}
\begin{theorem}[Uniform Risk Bound]
\label{thm:risk_bound}
Let $(P^{(i)}_{XY})_{i \in [N]}$ be $N$ training domains sampled i.i.d. from $\mu$, and let $S_i = (\x_{ij}, y_{ij})_{j \in [n]}$ be $n$-sample training sets for each domain. For any loss function $\ell$ that is bounded by $B_\ell$ and Lipschitz with constant $L_\ell$, we have, with probability at least $1-\delta$ for any $R$:
\begin{equation*}
\underset{f \in \cB_\kappa(R)}{\sup} |L(f, \infty) - \hat{L}_N(f, n)| \leq \cO\left(R\left(\frac{\log (N/\delta)}{n}^{\frac{\alpha}{2}} +  \sqrt{\frac{\log(1/\delta)}{N}}\right)\right).
\end{equation*}
\end{theorem}
\begin{proof}
We begin by decomposing the LHS.
\begin{align}
    \underset{f \in \cB_\kappa(R)}{\sup} |L(f, \infty) - \hat{L}_N(f, n)|  &\leq \underbrace{\underset{f \in \cB_\kappa(R)}{\sup} |L(f, \infty) - \hat{L}_N(f, \infty)|}_{A} + \underbrace{\underset{f \in \cB_\kappa(R)}{\sup} |\hat{L}_N(f, \infty) - \hat{L}_N(f, n)|}_{B}.
\end{align}
We first control $B$. Note that the loss function is Lipschitz in terms of $f$, and therefore we have,
\begin{align}
    \underset{f \in \cB_\kappa(R)}{\sup} |\hat{L}_N(f, \infty) - \hat{L}_N(f, n)| & \leq \frac{L_\ell}{nN} \cdot \underset{f \in \cB_\kappa(R)}{\sup} \left| \sum_{i=1}^N\sum_{j=1}^n{f(\x_{ij}, \bmu(D_i)) - f(\x_{ij}, \bmu(\widehat D_i)) }\right| \\
    & \leq \frac{L_\ell}{nN}  \sum_{i=1}^N\sum_{j=1}^n\underset{f \in \cB_\kappa(R)}{\sup} \left|f(\x_{ij}, \bmu(D_i)) - f(\x_{ij}, \bmu(\widehat D_i)) \right| \\
    & \leq \frac{L_\ell}{N}  \sum_{i=1}^N\underset{f \in \cB_\kappa(R)}{\sup} \left|f(\cdot, \bmu(D_i)) - f(\cdot, \bmu(\widehat D_i)) \right|.
\end{align}
We will now bound $\underset{f \in \cB_\kappa(R)}{\sup} \left|f(\cdot, \bmu(D_i)) - f(\cdot, \bmu(\widehat D_i)) \right|$. Note that by the reproducing property,
\begin{align}
    \underset{f \in \cB_\kappa(R)}{\sup} \left|f(\x, \bmu(D_i)) - f(\x, \bmu(\widehat D_i)) \right| &\leq \lVert f \rVert_{\kappa} \sup \left|f_\kappa(k_P(\bmu(D), \cdot), k_X(\x, \cdot))- f_\kappa(k_P(\bmu(\widehat D), \cdot), k_X(\x, \cdot))\right| \\
    &\leq R \cdot\sup \left|f_\kappa(k_P(\bmu(D), \cdot), k_X(\x, \cdot))- f_\kappa(k_P(\bmu(\widehat D), \cdot), k_X(\x, \cdot))\right| \tag{Since $\lVert f\rVert_\kappa \leq R$}\\
    &\leq RL_P \cdot\sup \left| \fK(\bmu(D), \cdot) - \fK(\bmu(\widehat D), \cdot) \right| \tag{Since $f_\kappa$ is Lipschitz} \\
    &\leq RL_P \cdot\sup \left\lVert \Phi_\fK(\bmu(D)) - \Phi_\fK(\bmu(\widehat D))\right\rVert \tag{Triangle inequality}\\
    &\leq RL_P \cdot\sup \left\lVert \bmu(D) - \bmu(\widehat D)\right\rVert_\infty^\alpha \tag{$\alpha$-H\"older assumption}
\end{align}
By Theorem 1, we can bound this term as well. Putting the results together and taking a union bound over all $N$ domains, we have with probability at least $1-\delta$,
\begin{align}
    \underset{f \in \cB_\kappa(R)}{\sup} |\hat{L}_N(f, \infty) - \hat{L}_N(f, n)| & \leq \frac{L_\ell L_P}{N} \left(R\frac{8\sigma^2\log(N/\delta)}{n}- \frac{1}{4}\right)^{\alpha/2}
\end{align}
Control of the term $B$ is done identically as Section 3.2 of the supplementary material in Blanchard~\etal\cite{blanchard2011generalizing}, with one key difference: in the control of term (IIa) (in their notation), we use the Lipschitz kernel instead of the product kernel, giving a constant $\sqrt{L_P}$ (instead of $B_\kappa$) in the bound of IIa. Combining the two steps gives the final stated result.
\end{proof}
Our proof admits identical dependences as Blanchard~\etal\cite{blanchard2011generalizing}, but the analysis differs in two key aspects: first, we use a Bernstein concentration to obtain a result in terms of the variance (Theorem 1 and term A of Theorem 2), which can potentially be tighter if the variance is low (see main paper and consistency). Next, we extend their result from product kernels to a general form of Lipschitz kernels, more suitable for deep-learning systems.
\section{Two-Step Training Procedure}
\begin{figure*}[htb]
\centering
\includegraphics[width=\linewidth]{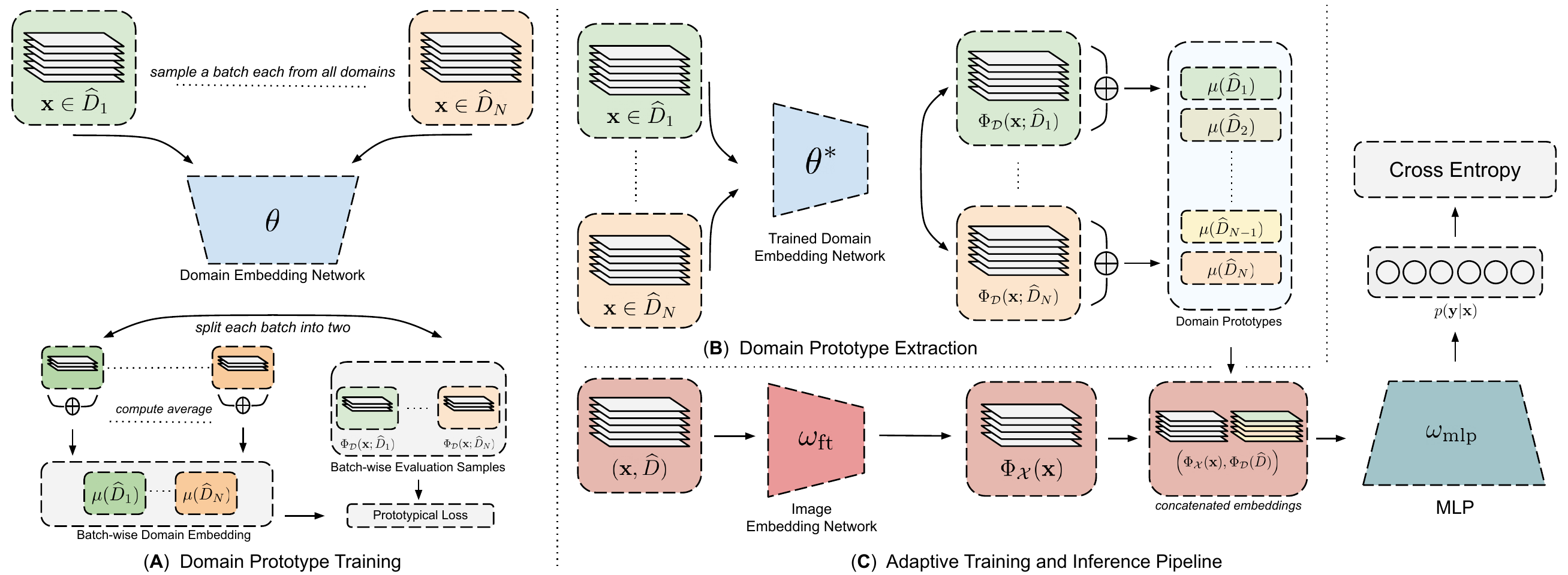}
\caption{Illustrative figure for the two-step training process.}\label{fig:net_architecture}
\end{figure*}

In our approach, we directly train a \textit{domain-aware} neural network $g: \bbR^d \times \bbR^{d_D} \rightarrow [K]$. For any input image $\x$ from domain $\cD$, $g$ takes in the \textit{augmented} input $(\x, \phi(\cD; \btheta^*))$, which is composed of the input $\x$ and the corresponding domain prototype $\phi(\cD; \btheta^*)$ obtained from the previous section, and predicts the class label $\hat y$ as output. The neural network architecture is described in Figure~\ref{fig:net_architecture}.

$g$ is a composition of an image feature extractor ($\Phi_X$) whose output is concatenated with the domain prototype $\Phi(\cD)$ and fed into a series of non-linear layers ($\fK$) to produce the final output. The domain-aware neural network parameters are denoted by $\bomega$. $f$ therefore is parameterized by both $\bomega$ and $\btheta$ and is described as $f(\x, \cD; \bomega, \btheta) = g(\x, \phi(\cD; \btheta); \bomega)$.

\begin{remark}
It is possible to decouple the prototype construction and inference procedures in the test phase (since we can use unsupervised samples from a domain obtained \textit{a priori} to construct prototypes). In this setting, we can use a distinct set of $n_p$ points to construct the domain prototype for each test domain. We can see directly from Theorem~\ref{lem:feature_variance} that for any Lipschitz loss function, the maximum disparity between classification using the optimal prototype (i.e., formed with knowledge of $\cD$) and the $n_p$-sample approximation is (with high probability) at most $\widetilde{\cO}\left(n_p^{-\alpha} + (\sigma_P/n_p)^{-\alpha/2}\right)$, which is small when $\sigma_P$ is small, i.e., the prototype is consistent.
\end{remark}

\section{Hyperparameters}

\subsection{Small-Scale Hyperparameters}
We follow the setup for small-scale datasets that is identical to Gulrajani and Lopez-Paz~\cite{gulrajani2020search} and use their default values, and the search distribution for each hyperparameter via random search. These values are summarized in Table \ref{table:hyperparameters}.

\begin{table}[H]
    \begin{center}
    { 
    \begin{tabular}{llll}
        \toprule
        \textbf{Condition} & \textbf{Parameter} & \textbf{Default value} & \textbf{Random distribution}\\
        \midrule
        \multirow{3}{*}{Basic hyperparameters}       & learning rate & 0.00005 & $10^{\text{Uniform}(-5, -3.5)}$\\
                                      & batch size    & 32   & $2^{\text{Uniform}(3, 5.5)}$\\
                                      & weight decay & 0    & $10^{\text{Uniform}(-6, -2)}$\\
        \midrule
        \multirow{8}{1.5cm}{C-DANN} & lambda                 & 1.0    & $10^{\text{Uniform}(-2, 2)}$\\
                                      & generator learning rate & 0.00005 & $10^{\text{Uniform}(-5, -3.5)}$\\
                                      & generator weight decay & 0    & $10^{\text{Uniform}(-6, -2)}$\\
                                      & discriminator learning rate & 0.00005 & $10^{\text{Uniform}(-5, -3.5)}$\\
         & discriminator weight decay & 0    & $10^{\text{Uniform}(-6, -2)}$\\
         & discriminator steps        & 1    & $2^{\text{Uniform}(0, 3)}$\\
         & gradient penalty           & 0    & $10^{\text{Uniform}(-2, 1)}$\\
         & adam $\beta_1$             & 0.5    & $\text{RandomChoice}([0, 0.5])$\\
        \midrule
        \multirow{2}{*}{IRM}          & lambda  & 100    & $10^{\text{Uniform}(-1, 5)}$\\
                                      & iterations of penalty annealing & 500 & $10^{\text{Uniform}(0, 4)}$\\
        \midrule
        Mixup                         & alpha & 0.2 & $10^{\text{Uniform}(0, 4)}$\\
        \midrule
        DRO                           & eta   & 0.01 & $10^{\text{Uniform}(-1, 1)}$\\
        \midrule
        MMD                           & gamma & 1 & $10^{\text{Uniform}(-1, 1)}$\\
        \midrule
        MLDG           & beta & 1 & $10^{\text{Uniform}(-1, 1)}$\\
        \midrule
        all          & dropout & 0    & $\text{RandomChoice}([0, 0.1, 0.5])$\\
        \bottomrule
    \end{tabular}
    }
    \end{center}
    \caption{Hyperparameters for small-scale experiments.} 
    \label{table:hyperparameters}
\end{table}
\subsubsection{Prototypical Network}
In addition to these, the prototypical network optimal hyperparameters are as follows: We use a ResNet-50 architecture initialized with ILSVRC12~\cite{deng2009imagenet} weights, available in the PyTorch Model Zoo. We run 1000 iterations of prototypical training with $N_t = 4$ domains sampled each batch, and a batch size of $32$ per domain. The learning rate is $1e-6$ and weight decay is $1e-5$ with $50\%$ of the points in each batch used for classification and the rest for creating the prototype.
\subsubsection{DA-CORAL and DA-MMD}
The DA-CORAL and DA-MMD architectures are identical to DA-ERM. We additionally add the MMD and CORAL regularizers with $\gamma = 1$ for each in the final penultimate layer of the network (after the bottleneck, before the logits).
\subsection{Large-Scale Hyperparameters}
\subsubsection{LT-ImageNet}
\textbf{DA-ERM Prototype Training}. We did a small grid-search for learning rate (in the set $1e-1, 1e-2, 1e-3, 1e-4, 1e-5$) and the final parameters are: learning rate of $0.001$ and weight decay of $1e-5$. We train with a batch size of 100 (per domain), over 8 GPUs and run the prototypical algorithm for 1300 iterations. We use $50\%$ of the points for the support set per batch per domain and the remaining as test points. 

\textbf{Main Training for all methods}. For all methods (MMD, CORAL, ERM, DA-ERM) we perform main training for 5 epochs of the training data with a batch-size of 100 (per domain), learning rate of $0.005$ and weight-decay of $1e-5$ over a system of 8 GPUs. We additionally tune the value of the loss weight in the range $(0, 5)$ for MMD and CORAL. The reported results are with the loss weight ($\gamma = 1$).

\subsubsection{Geo-YFCC}
\textbf{DA-ERM Prototype Training}. We did a small grid-search for learning rate (in the set $1e-2, 1e-3, 1e-4, 1e-5$), weight-decay (in the set $1e-4, 1e-5, 1e-6$), and the final parameters are: learning rate of $2e-2$ and weight decay of $1e-5$. We train with a batch size of 80 (per domain), over 48 GPUs and run the prototypical algorithm for 1300 iterations. We use $80\%$ of the points for the support set per batch per domain and the remaining as test points. 

\textbf{Main Training for all methods}. We perform a grid search on number of epochs of training data in the range $(3, 6, 12, 25)$ and found that all methods performed best (overfit the least to training domains) at 6 epochs. For all methods (MMD, CORAL, ERM, DA-ERM) we perform main training with a batch-size of 80 (per domain), learning rate of $0.04$ and weight-decay of $1e-5$ over a system of 64 GPUs. We additionally tune the value of the loss weight in the range $(0, 1.5)$ for MMD and CORAL. The reported results are with the loss weight ($\gamma = 1$).

\section{Consistency Experiment}

\begin{table}[th!]
    \rowcolors{3}{blue!5}{white}
    \centering
    \begin{tabular}{ccccccc}
    \hline
    \hline\multirow{2}{*}{Dataset} & \multicolumn{5}{c}{Accuracy on Validation Set, top1/top5} \\
                  & 50 & 100   & 250   & 500   & 1000   & 2000 \\
    \hline
    DG - ImageNet & -- & 56.0/80.1  & 56.1/80.2  & 56.2/80.2  & --   & --   \\
    Geo - YFCC    & 23.7/49.0 & 23.3/49.1  & 23.4/49.3  & 23.6/49.3  & 23.4/49.1   & 23.4/49.2   \\
    \hline
    \hline
    \end{tabular}
    \caption{Number of points used to construct prototype.}
    \label{tab:exp_abl_num_proto_points}
\end{table}
Table~\ref{tab:exp_abl_num_proto_points} shows the effect of varying the number of data points used to consruct the domain prototypes for LT-ImageNet and Geo-YFCC datasets. We observe that performance remains similar till $50$ points. This is desirable as in many settings we do not have access to many samples from new domains.
{\small
\bibliographystyle{ieee_fullname}
\bibliography{egbib}
}

\end{document}